\documentclass[letterpaper]{article} 
\usepackage[submission]{aaai24}  
\usepackage{times} 
\usepackage{helvet} 
\usepackage{courier}  
\usepackage[hyphens]{url}  
\usepackage{graphicx}
\usepackage{natbib}  
\usepackage{caption} 
\usepackage{algorithm}
\usepackage{algorithmic}
\usepackage{newfloat}
\usepackage{listings}
\DeclareCaptionStyle{ruled}{labelfont=normalfont,labelsep=colon,strut=off} 
\floatstyle{ruled}
\newfloat{listing}{tb}{lst}{}
\floatname{listing}{Listing}
\usepackage{cancel}
\usepackage{bm}
\usepackage{microtype}
\usepackage[utf8]{inputenc}
\usepackage{mathrsfs}
\usepackage{amsthm,amsfonts,amscd,amsmath,amssymb}
\usepackage[english]{babel}
\usepackage{mathtools}
\usepackage{libertine}
\usepackage{libertinust1math}
\usepackage{enumitem}
\usepackage{tikz}
\usepackage{xcolor}
\usepackage{microtype}
\usepackage{accents}
\usetikzlibrary{cd}
\usepackage{enumerate}
\usepackage{stmaryrd}

\usepackage{bbold}

\newcommand{\loss}{\mathcal{L}}

\DeclareMathOperator{\Tr}{Tr}

\DeclarePairedDelimiter{\abs}{\lvert}{\rvert}
\DeclareMathOperator*{\argmin}{arg min}

\newcommand{\st}{\textup{s.t.}}

\theoremstyle{plain}

  \newtheorem{proposition}{Proposition}
  \newtheorem{lemma}{Lemma}
  \newtheorem{theorem}{Theorem}
  
  \newtheorem{conjecture}{Conjecture}
\theoremstyle{definition}
  \newtheorem{definition}{Definition}

  \newtheorem{example}{Example}
  
  \newtheorem{remark}{Remark}

\newcommand{\calL}{\mathcal{L}}

\newcommand{\calN}{\mathcal{N}}

\newcommand{\bbE}{\mathbb{E}}

\newcommand{\bbI}{\mathbb{I}}

\newcommand{\bbR}{\mathbb{R}}

\newcommand{\bu}{\boldsymbol{u}}

\newcommand{\bI}{\boldsymbol{I}}

\newcommand{\norm}[1]{\left\lVert#1\right\rVert}

\newcommand{\row}{\text{row}}

\usepackage{todonotes}

\title{Representation Learning Dynamics of Self-Supervised Models}

\author {
    Pascal M. Esser\equalcontrib\textsuperscript{\rm 1},
    Satyaki Mukherjee\equalcontrib\textsuperscript{\rm 1}
    Debarghya Ghoshdastidar\textsuperscript{\rm 1}
    }
\affiliations {
    \textsuperscript{\rm 1}Technical University of Munich, Germany\\
    esser@cit.tum.de, mukherjee@cit.tum.de, ghoshdas@cit.tum.de
}

\begin{document}

\maketitle

\begin{abstract}
Self-Supervised Learning (SSL) is an important paradigm for learning representations from unlabelled data, and SSL with neural networks has been highly successful in practice. However current theoretical analysis of SSL is mostly restricted to generalisation error bounds. In contrast, learning dynamics often provide a precise characterisation of the behaviour of neural networks based models but, so far, are mainly known in supervised settings. In this paper, we study the learning dynamics of SSL models, specifically representations obtained by minimising contrastive and non-contrastive losses. We show that a n{\"a}ive extension of the dymanics of multivariate regression to SSL leads to learning trivial scalar representations that demonstrates dimension collapse in SSL. Consequently, we formulate SSL objectives with orthogonality constraints on the weights, and derive the exact (network width independent) learning dynamics of the SSL models trained using gradient descent on the Grassmannian manifold. We also argue that the infinite width approximation of SSL models significantly deviate from the neural tangent kernel approximations of supervised models. We numerically illustrate the validity of our theoretical findings, and discuss how the presented results provide a framework for further theoretical analysis of contrastive and non-contrastive SSL.
\end{abstract}

\section{Introduction}

A common way to distinguish between learning approaches is to categorize them into unsupervised learning, which relies on a input data consisting of a feature vector $(x)$, and supervised learning which relies on feature vectors and corresponding labels $(x , y)$. However, in recent years, Self-Supervised Learning (SSL) has been established as an important paradigm between supervised and unsupervised learning as it does not require explicit labels but relies on implicit knowledge of what makes some samples semantically close to others. Therefore SSL builds on inputs and inter-sample relations $(x , x^+)$, where $x^+$ is often constructed through data-augmentations of $x$ known to preserve input semantics such as additive noise or horizontal flip for an image \cite{kanazawa2016warpnet, novotny2018self, gidaris2018unsupervised}. While the idea of SSL is not new  \cite{bromley1993signature}, recent deep SSL models have been highly successful in computer vision \cite{chen2020simple,caron2021emerging,Jing2019SelfSupervisedVF}, 
natural language processing \cite{ misra2020self,BERT2019}, speech recognition \cite{Steffen2019arxiv,Mohamed2022}.
Since the early works \cite{bromley1993signature}, methods for SSL have predominantly relied on neural networks however with a strong focus on model design with only little theoretical backing. 

The main focus of the theory literature on SSL has been either on providing generalization error bounds for downstream tasks on embeddings obtained by SSL \citep{Arora2019ATA,GeTFJ-2303-01566,BansalKB21,LeeLSZ21,SaunshiMA21,ToshK021,WeiXM21,0002NN22,ChenZXCD0TZC22}, or analysing the spectral and isoperimetric properties of data augmentation \citep{BalestrieroL22,HanYZ23,Zhuo0M023}. The latter approach also result in novel bounds on the generalisation error \citep{HaoChen2021ProvableGF,abs-2306-00788}. 
While generalisation theory remains one of the fundamental tools to characterise the statistical performance, it has been already established for supervised learning that classical generalisation error bounds do not provide a complete theoretical understanding and can become trivial in the context of neural network models \citep{ZhangBHRV17,NeyshaburBMS17}. 
Therefore a key focus in modern deep learning theory is to understand the learning dynamics of models, often under gradient descent, as they provide a more tractable expression of the problem that can be an essential tool to understand the loss landscape and convergence \citep{fukumizu1998dynamics,SaxeMG13,Pretorius2018LearningDO}, early stopping \citep{li2021implicit}, linearised (kernel) approximations \citep{Jacot2018Neurips,Du_2019_NEURIPS2019} and, mostly importantly, generalisation and inductive biases \citep{SoudryHNGS18,luo2018towards,HeckelY21}.

\emph{In this paper, we analyze the learning dynamics of SSL models under contrastive and non-contrastive losses \citep{Arora2019ATA,chen2020simple}, which we show to be significantly different from the dynamics of supervised models.} This gives a simple and precise characterization of the dynamics that can provide the foundation for future theoretical analysis of SSL models. Before presenting the learning dynamics, we recall the SSL principles and losses consisdered in this work.

\textbf{Contrastive Learning.}
Contrastive SSL has its roots in the work of \citet{bromley1993signature}. Recent deep learning based contrastive SSL show great empirical success in computer vision \citep{chen2020simple,caron2021emerging,Jing2019SelfSupervisedVF}, video data \citep{fernando2017self,sermanet2018time}, natural language tasks \citep{ misra2020self,BERT2019} and speech \citep{Steffen2019arxiv,Mohamed2022}. 
In general a contrastive loss is defined by considering an anchor image, $x \in \mathbb{R}^d$, positive samples $\{x^+\} \subset \mathbb{R}^d$ generated using data augmentation techniques as well as independent negative samples $\{x^-\} \subset \mathbb{R}^d$. The heuristic goal is to align the anchor more with the positive samples than the negative ones, which is rooted in the idea of maximizing mutual information between similar samples of the data. 
In this work, we consider a simple contrastive loss minimisation problem along the lines of \citet{Arora2019ATA}, assuming exactly one positive sample $x_i^+$ and one negative sample $x_i^-$ for each anchor $x_i$,\footnote{It is straightforward to extend our analysis to multiple positive and negative samples, but the expressions become cumbersome, without providing additional insights.}
 \begin{align}
    &\min_{\Theta} \sum^n_{i=1}  u(x_i)^\top\left( u(x_i^-) -  u(x_i^+)\right) \label{eq:contrastive loss function},
\end{align}
where $ u = [u_1(\cdot,\Theta) \ldots u_z(\cdot,\Theta)]^\top:\mathbb{R}^d\to\mathbb{R}^z$ is the embedding function, parameterized by $\Theta$, the learnable parameters.

\textbf{Non-Contrastive Learning}
Non-contrastive losses emerged from the observation that negative samples (or pairs) in contrastive SSL are not necessary in practice, and it suffices to maximise only alignment between positve pairs \cite{chen2020simsiam,chen2020simple,Grill2020Neurips}. 
 Considering a simplified version of the setup in  \cite{chen2020simple}
 one learns a representation by minimising the loss
 \footnote{We simplify \citet{chen2020simple} by replacing the cosine similarity with the standard dot product and also by replacing an additional positive sample $x_i^{++}$ by anchor $x_i$ for convenience.}
\begin{align}
    &\min_{\Theta} \sum^n_{i=1} -  u(x_{i})^\top  u(x_i^{+}).\label{eq:non-contrastive loss function}
\end{align}
The embedding  $ u = [u_1(\cdot,\Theta)\ldots u_z(\cdot,\Theta)]^\top:\mathbb{R}^d\to\mathbb{R}^z$, parametrised by $\Theta$, typically comprises of a {base encoder network}  and a {projection head} in practice \citep{chen2020simple}.

\textbf{Contributions.} 
The objective of this paper is to derive the evolution dynamics of the learned embedding $u=u(\cdot,\theta)$ under gradient flow for the constrastive \eqref{eq:contrastive loss function} and non-contrastive losses \eqref{eq:non-contrastive loss function}.
More specifically we show the following:
\begin{itemize}
     \item We express the learning dynamics for both contrastive and non-contrastive learning and show that, the evolution dynamics is same across dimensions. This explains why SSL is naturally prone to dimension collapse.      
     \item Assuming a 2-layer linear network, we show that dimension collapse cannot be avoided by adding standard Frobenius norm reguralisation or constraint, but by adding orthogonality or L2 norm constraints.
     \item We further show that at initialization, the dynamics of 2-layer network with nonlinear activation is close to their linear, width independent counterparts (Theorem~\ref{th:non-linear}). We also provide empirical evidence that the evolution of the infinite width non-linear networks are close to their linear counterparts, under certain conditions on the nonlinearity (that hold for tanh).
    \item  We derive the learning dynamics of SSL for linear networks, under orthogonality constraints (Theorem~\ref{th: linear dynamics}). 
    We further show the convergence of the learning dynamics for the one dimensional embeddings $(z=1)$.
    \item We numerically show, on the MNIST dataset, that our derived SSL learning dynamics can be solved significantly faster than training nonlinear networks, and yet provide comparable accuracy on downstream tasks.   
\end{itemize}

All proofs are provided in the appendix.

\textbf{Related works.}
Our focus is on the evolution of the learned representations, and hence, considerably different from the aforementioned literature on generalisation theory and spectral analysis of SSL. From an optimisation perspective, \citet{LiuLUT23} derive the loss landscape of contrastive SSL with linear models, $u(x) = Wx$, under InfoNCE loss \cite{abs-1807-03748}. 
Although the contrastive loss in \eqref{eq:contrastive loss function} seems simpler than InfoNCE, they are structurally similar under linear models \citep[see Eqns. 4--6]{LiuLUT23}.
Training dynamics for contrastive SSL with deep linear models have been partially investigated by \citet{Tian22}, who show an equivalence with principal component analysis, and by \citet{JingVLT22}, who establish that dimension collapse occurs for over-parametrised linear contrastive models.
Theorem \ref{th: linear dynamics} provides a more precise characterisation and  convergence criterion of the evolution dynamics than previous works. Furthermore, none of prior works consider non-linear models or orthogonality constraints as studied in this work.

We also distinguish our contributions (and discussions on neural tangent kernel connections) with the kernel equivalents of SSL studied in \citet{abs-2209-14884,0001HM23,ShahSCC22,abs-2302-02774}. While \citet{ShahSCC22,abs-2302-02774} specifically pose SSL objectives using kernel models, \citet{abs-2209-14884,0001HM23} show that contrastive SSL objectives induce specific kernels. Importantly, these works neither study the learning dynamics nor consider the neural tangent kernel regime. 

\textbf{Notation.}
Let $\bbI_n$ be an $n\times n$ identity matrix. For a matrix $A$ let $\norm{A}_F$ and $\norm{A}_2$ be the standard frobenious norm and the $L2$-operator norm respectively. 
The machine output is denoted by $u(\cdot)$. While $u$  is time dependent and should be more accurately denoted as $u_t$ we suppress the subscript where obvious. For any time dependent function, for instance $u$, denote $\mathring{u}$ to be its time derivative i.e. $\frac{d u_t}{dt}$.
$\phi$ is used to denote our non-linear activation function and we abuse notation to also denote its co-ordinate-wise application on a vector by $\phi( \cdot )$.
$\langle \cdot,\cdot \rangle$ is used to denote the standard dot product.

\section{Learning Dynamics of Regression and its N{\"a}ive Extension to SSL}
\label{sec: contrastive}

In the context of regression, \citet{Jacot2018Neurips} show that the evolution dynamics of (infinite width) neural networks, trained using gradient descent under a squared loss, is equivalent to that of specific kernel machines, known as the neural tangent kernels (NTK).
The analysis has been extended to a wide range of models, including convolutional networks \citep{Arora20219NeurIPS}, recurrent networks \citep{Alemohammad0BB21}, overparametrised autoencoders \citep{Nguyen_2021_IEEE},  graph neural networks \citep{Du_2019_NEURIPS2019,Sabanayagam2022} among others.
However, these works are mostly restricted to squared losses, with few results for margin loss \citep{ChenHNW21}, but derivation of such kernel machines are still open for contrastive or non-contrastive losses \eqref{eq:contrastive loss function}--\eqref{eq:non-contrastive loss function}, or broadly, in the context of SSL.
To illustrate the differences between regression and SSL, we outline the learning dynamics of multivariate regression with squared loss, and discuss how a n{\"a}ive extension to SSL is inadequate.

\subsection{Learning Dynamics of Multivariate Regression}

Given a training feature matrix $X:=\left[x_1,\cdots,x_n\right]^\top\in\bbR^{n\times d}$ and corresponding $z$-dimensional labels $Y:=\left[y_1,\cdots,y_n\right]^\top \in\bbR^{n\times z} $, consider the regression problem of learning a neural network function $u(x) = [u_1(x,\Theta) \ldots u_z(x,\Theta)]^\top$, parameterized by $\Theta$, by minimising the squared loss function
\begin{align*}
    \calL (\Theta):=\frac{1}{2}\sum_{i=1}^n \Vert u(x_i) - y_i\Vert^2.
\end{align*}
Under gradient flow, the evolution dynamics of the parameter during training is $\mathring{\Theta} = -\nabla_{\Theta}\loss$ and, consequently, the evolution of the $l$-th component of network output $\bu(x)$, for any input $x$, follows the differential equation
\begin{align}
\mathring{u}_l(x) 
&= \left\langle \nabla_\Theta u_l(x), \mathring{\Theta}\right\rangle \nonumber
\\&= - \sum_{i=1}^n\sum_{j=1}^z
\left\langle\nabla_\Theta u_l(x), \nabla_\Theta u_j(x_i)\right\rangle
(u_j(x_i) - y_{i,j}) .
\label{eq:regr-dynamics}
\end{align}

While the above dynamics apparently involve interaction between the different dimensions of the output $u(x)$, through $\left\langle\nabla_\Theta u_l(x), \nabla_\Theta u_j(x_i)\right\rangle$, it is easy to observe that this interaction does not contribute to the dynamics of linear or kernel models. We formalise this in the following lemma.

\begin{lemma}[\textbf{No interaction across output dimensions}]
\label{lem:dimension-interaction}
Let $u:\bbR^d\to\bbR^z$ be either a linear model $u(x) = \Theta x$, or a kernel machine $u(x) = \Theta \psi(x)$, where $\psi$ corresponds to the implicit feature map of a kernel $k$, that is, $k(x,x') = \langle \psi(x),\psi(x')\rangle$.
\\
Then in the infinite width limit ($h\rightarrow \infty$) the inner products between the gradients are given by
\begin{align*}
\left\langle\nabla_\Theta u_l(x), \nabla_\Theta u_j(x')\right\rangle = \left\{ \begin{array}{ll}
0 & \text{if } l \neq j,\\
x^\top x' & \text{if } l=j \text{ (linear case)},\\
k(x, x') & \text{if } l=j \text{ (kernel case)}.\end{array}\right.
\end{align*}
\end{lemma}

For infinite width neural networks, whose weights are randomly initialised with appropriate scaling,  \citet{Jacot2018Neurips} show that at, initialisation, Lemma \ref{lem:dimension-interaction} holds with $k$ being the neural tangent kernel. Approximations for wide neural networks further imply the kernel remains same during training \citep{0001ZB20}, and so Lemma \ref{lem:dimension-interaction} continues to hold through training.

\begin{remark}[\textbf{Multivariate regression $=$ independent univariate regressions}]
A consequence of Lemma \ref{lem:dimension-interaction} is that the learning dynamics \eqref{eq:regr-dynamics} simplifies to
\begin{align*}
\mathring{u}_l(x) 
&= - \sum_{i=1}^n
\left\langle\nabla_\Theta u_l(x), \nabla_\Theta u_l(x_i)\right\rangle
(u_l(x_i) - y_{i,l}),
\end{align*}
that is, each component of the output $u_l$ evolves independently from other $u_j, j\neq l$.
Hence, one may solve a $z$-variate squared regression problem as $z$ independent univariate problems.
We discuss below that a similar phenomenon is true in SSL dynamics with disastrous consequences.  
\end{remark}

\subsection{Dynamics of n{\"a}ive SSL has Trivial Solution}

We now present the learning dynamics of SSL with contrastive and non-contrastive losses in \eqref{eq:contrastive loss function}--\eqref{eq:non-contrastive loss function}.
For convenience, we first discuss the non-contrastive case. Assuming that the network function $u:\bbR^d\to\bbR^z$ is parametrised by $\Theta$, the gradient of the loss $\calL (\Theta) = \sum\limits_{i=1}^n - u(x_i)^\top u(x_i^+)$ is 
\begin{align*}
    \nabla_\Theta \calL (\Theta) = - \sum_{i=1}^n\sum_{j=1}^z u_j(x_i) \cdot \nabla_\Theta u_j(x_i^+) + u_j(x_i^+) \cdot \nabla_\Theta u_j(x_i)
\end{align*}
Hence, under gradient descent $\mathring{\Theta} = -\nabla_\Theta \calL$, the evolution of each component of $u(x)$, given by $\mathring{u}_l(x) = \left\langle \nabla_\Theta u_l(x), \mathring{\Theta}\right\rangle$ is
\begin{align}
\mathring{u}_l(x)
= \sum_{i=1}^n\sum_{j=1}^z
&\left\langle\nabla_\Theta u_l(x), \nabla_\Theta u_j(x_i)\right\rangle 
u_j(x_i^{+}) 
\nonumber
\\+ &\left\langle\nabla_\Theta u_l(x), \nabla_\Theta u_j(x_i^{+})\right\rangle 
u_j(x_i).
\label{eq:non-contrastive dynamics}
\end{align}
Similarly, in the case of contrastive loss \eqref{eq:contrastive loss function}, the learning dynamics of $u(x)$, for any input $x$, is similarly expressed by
\begin{align}
\mathring{u}_l(x) = \sum_{i=1}^n\sum_{j=1}^z
&\left\langle\nabla_\Theta u_l(x), \nabla_\Theta u_j(x_i)\right\rangle 
\nonumber
u_j(x_i^{+}) 
\\&+ \left\langle\nabla_\Theta u_l(x), \nabla_\Theta u_j(x_i^{+})\right\rangle 
u_j(x_i)
\nonumber\\
&-\left\langle\nabla_\Theta u_l(x), \nabla_\Theta u_j(x_i)\right\rangle 
u_j(x_i^{-}) 
\nonumber
\\&- \left\langle\nabla_\Theta u_l(x), \nabla_\Theta u_j(x_i^{-})\right\rangle 
u_j(x_i).
\label{eq:contrastive dynamics}
\end{align}

We note Lemma \ref{lem:dimension-interaction} depends only on the model and not the loss function, and hence, it is applicable for the SSL dynamics in \eqref{eq:non-contrastive dynamics}--\eqref{eq:contrastive dynamics}. However,  there are no multivariate training labels $y \in \bbR^z$ in SSL (i.e. $y=0$) that can drive the dynamics of the different components $u_1,\ldots,u_z$ in different directions, which leads to dimension collapse.

\begin{proposition}[\textbf{Dimension collapse in SSL dynamcis}]
\label{prop:dimension-collapse}
    Under the conditions of Lemma \ref{lem:dimension-interaction}, every component of the network output $u:\bbR^d\to\bbR^z$ has identical dynamics, and hence, identical fixed points. As a consequence, the output collapses to one dimension at convergence. 
    \\
    For linear model, $u(x) = \Theta x$, the dynamics of $u(x)$ is given by
    \begin{align*}
        \mathring{u}_l(x) = \sum_{i=1}^n (x^\top x_i) u_l(x_i^+) + (x^\top x_i^+) u_l(x_i)
    \end{align*}
    for the non-contrastive case, and
    \begin{align*}
        \mathring{u}_l(x) = \sum_{i=1}^n (x^\top x_i) \big(u_l(x_i^+)-u_l(x_i^-)\big) 
        + (x^\top x_i^+ - x^\top x_i^-) u_l(x_i)
    \end{align*}
    for the contrastive case. For kernel models, the dynamcis is similarly obtained by replacing each $x^\top x'$ by $k(x,x')$.
\end{proposition}

By the extension of Lemma \ref{lem:dimension-interaction} to neural network and NTK dynamics, one can conclude that Proposition \ref{prop:dimension-collapse} and dimension collapse also happen for wide neural networks, when trained for the SSL losses in \eqref{eq:contrastive loss function}--\eqref{eq:non-contrastive loss function}.

\begin{remark}[\textbf{SSL dynamics for other losses}]
   One may argue that the above dimension collapse is a consequence of loss definitions in \eqref{eq:contrastive loss function}--\eqref{eq:non-contrastive loss function}, and may not exist for other losses. We note that \citet{LiuLUT23} analyse contrastive learning with linear model under InfoNCE, and the simplified loss closely resembles \eqref{eq:contrastive loss function}, which implies decoupling of output dimensions (and hence, dimension collapse) would also happen for InfoNCE. The same argument also holds for non-constrastive loss in \citet{chen2020simple}. However, for the spectral contrastive loss of \citet{HaoChen2021ProvableGF}, the output dimensions remain coupled in the SSL dynamics due to existing interactions $u(x_i)^\top u(x_i^-) $ on the training data.  
\end{remark}

\begin{remark}[\textbf{Projections cannot overcome dimension collapse}]
    \citet{JingVLT22} propose to project the representation learned by a SSL model into a much smaller dimension, and show that fixed (non trainable) projectors may suffice. For a linear model, this implies $u(x) = A\Theta x$, where $A \in \bbR^{r\times z}, r\ll z$ is fixed. It is straightforward to adapt the dynamics and Proposition \ref{prop:dimension-collapse} to this case, and observe that for any $r>1$, all the $r$ components of $u(x)$ have identical learning dynamics, and hence, collapse at convergence.
\end{remark}

\section{SSL with (Orthogonality) Constraints}\label{sec: constraints}

For the remainder of the paper, we assume that the SSL model $u:\bbR^d\to\bbR^z$ corresponds to a 2-layer neural network of the form
\begin{align*}
    x\in\bbR^{d}
    \ \xrightarrow{W_1}{} \ \bbR^{h}
     \ \xrightarrow{\phi( \cdot )}{} \ \bbR^{h}
    \  \xrightarrow{W_2^\top}{} \ u(x) = W_2^\top \phi(W_1 x)\in \bbR^z,
\end{align*}
where $h$ is the size of the hidden layer and $\Theta = (W_1,W_2^\top)$ are trainable matrices. 
Whenever needed, we use $u^\phi$ for the output to emphasize the nonlinear activation $\phi$, and contrast it with a 2-layer linear network $u^\bbI(x) = W_2^\top W_1x$.

Based on the discussion in the previous section, it is natural to ask how can the SSL problem be rephrased to avoid dimension collapse. An obvious approach is to add regularisation or constraints \citep{bardes2021vicreg,ermolov2021whitening,caron2020unsupervised}.
The most obvious regularisation or constraint on $W_1,W_2$ is entry-wise, such as on Frobenius norm. While there has been little study on various regularisations in SSL literature, a plethora of variants for Frobenius norm regularisations can be found for autoencoders, such as sum-regularsiation, $\Vert W_1\Vert_F^2+\Vert W_2 \Vert_F^2$, or product regularisation $\Vert W_2^\top W_1\Vert_F^2$ \citep{kunin19apmlr}. 

It is known in the optimisation literature that regularised loss minimisation can be equivalently expressed as constrained optimisation problems. In this paper, we use the latter formulation for convenience of the subsequent analysis.
The following result shows that Frobenius norm constraints do not prevent the output dimensions from decoupling, and hence, it is still prone to dimension collapse.

\begin{proposition}[\textbf{Frobenius norm constraint does not prevent dimension collapse}]
    Consider a linear SSL model $u^{\phi}(x) = W_2^\top \phi(W_1x)$. The optimisation problem
    \begin{align*}
        \min_{W_1,W_2} \calL (W_1, W_2) \quad \text{s.t.} \quad  \Vert W_1\Vert_F \leq c_1, \Vert W_2\Vert_F \leq c_2,
    \end{align*}
    where the loss $\calL$ is given by \eqref{eq:contrastive loss function} or \eqref{eq:non-contrastive loss function}, has a global solution $u(x) = [a(x) ~ 0 \ldots 0]^\top \in \bbR^z$. 
\end{proposition}

The above result precisely shows dimension collapse for linear networks $u^\bbI$ even with Frobenius norm constraints.
An alternative to Frobenius norm constraint can be to constrain the $L2$-operator norm. To this end, the following result shows that, for linear networks, the operator norm constraint can be realised in multiple equivalent ways.

\begin{proposition}[\textbf{Equivalence of operator norm and orthonognality constraints}]\label{prop:L2-ortho-equivalence}
Consider a linear SSL model $u^\bbI(x) = W_2^\top W_1x$, and let the loss $\loss(W_1,W_2)$ be given by either \eqref{eq:contrastive loss function} or \eqref{eq:non-contrastive loss function} whose general form is  $\loss(W_1,W_2) = \norm{W_2^\top W_1 C W_1^\top W_2}_2^2$, where $C$ has atleast one negative eigenvalue.
Then the following optimisation problems are equivalent:
\begin{align*}
    1. &\quad\min_{W_1,W_2} \frac{\calL (W_1, W_2)}{\norm{W_2}_2^2\norm{W_1}_2^2};\\
    2. &\quad\min_{W_1,W_2} \calL (W_1, W_2) \quad \text{s.t.} \quad  \Vert W_2\Vert_2 \leq 1, \ \Vert W_1\Vert_2 \leq 1;
    \\
    3. &\quad\min_{W_1,W_2} \calL (W_1, W_2) \quad \text{s.t.} \quad  \Vert W_2^\top W_1\Vert_2 \leq 1;\\
    4. &\quad\min_{W_1,W_2} \text{$\calL$} (W_1, W_2) \quad \text{s.t.} \quad  W_2^\top W_2 = {\bbI_z}, \ W_1^\top W_1 = {\bbI_d}.
\end{align*}
Additionally this regularization avoids dimension collapse.
\end{proposition}
\begin{figure}[t]
    \centering
    \includegraphics[width = 0.47\textwidth]{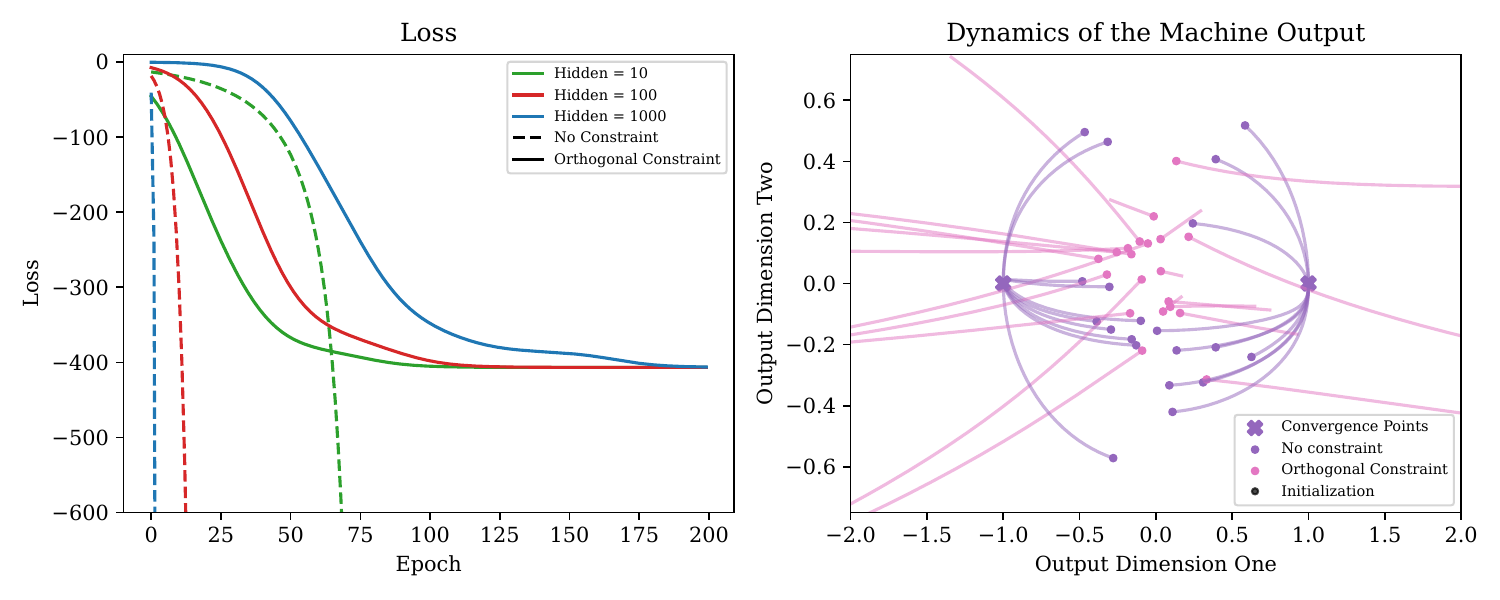}
    \caption{Comparison of gradient decent optimization with different regularisers \textbf{(left)} comparison of the loss function \textbf{(right)} comparison of the evolution of the outputs for the different considered constraints.
    }
    \label{fig:constraint comparison}
\end{figure}

Avoidance of dimensional collapse is also heuristically evident in the orthogonality constraint $W_2^\top W_2 = \bbI_z, \ W_1^\top W_1 = \bbI_d$, which we focus on in the subsequent sections. In particular we observe from the proof of Prop~\ref{prop:L2-ortho-equivalence} that this regularization extracts the eigenvectors of $C$ corresponding to its "most-negative" eigenvalues

\begin{example}[\textbf{SSL dynamics on half moons}]
We numerically illustrate the importance of constraints in SSL. We consider a contrastive setting (loss in \eqref{eq:contrastive loss function}) for the \emph{half moon} dataset \cite{scikit-learn}, where $x^-$ is an independent sample from the dataset and $x^+ = x + \varepsilon$ where $\varepsilon\sim\calN(0,0.1\bbI)$. 
Let us now compare the dynamics of $\calL$ (no constraints) and $ \calL_{orth}$, the scaling loss that corresponds to orthogonality constraints, and present the results in Figure~\ref{fig:constraint comparison}. We observe that under orthogonal constraints, independent of the initialization the function converges to fixed points (which we  theoretically show in Theorem~\ref{th: convergence}). On the other hand the dynamics for unconstrained loss $\calL$ diverge.   
\end{example}

\subsection{Non-Linear SSL Models are Almost Linear}

While the above discussion pertains to only linear models, we now show that the network, with nonlinear activation $\phi$ and orthognality constraints,
\begin{align*}
     u_{(t)}^{\phi}(x) &= W_2^\top \phi(W_1x)      
     &\st  ~ W_2^\top W_2 = \bbI_z, \ W_1^\top W_1 = \bbI_d,
\end{align*}
is almost linear. For this discussion, we explicitly mention the time dependence as a subscript $u_{(t)}^{\phi}$. 
We begin by arguing theoretically that in the infinite width limit at initialization there is very little difference between the output of the non-linear machine $u_{(0)}^{\phi}$ and that of its linear counterpart $u_{(0)}^{\bbI}$.

\begin{theorem}[\textbf{Comparison of Linear and Non-linear Network}]\label{th:non-linear}
Recall that  $u_{(t)}$ provides the output of the machine at time $t$ and therefore consider the linear and non-linear setting at initialization as
\begin{align}\label{eq: Th 2 machine u}
    u_{(0)}^{\bbI} &= W_2^\top W_1x
     &\st  ~ W_2^\top W_2 = \bbI_z, \ W_1^\top W_1 = \bbI_d;\\
    u_{(0)}^{\phi} &= W_2^\top \phi\left(W_1x\right)\nonumber
    &\st  ~ W_2^\top W_2 = \bbI_z, \ W_1^\top W_1 = \bbI_d\nonumber.
\end{align}
Let $\phi( \cdot )$ be an activation function, such that $\phi(0) = 0$,  $\phi'(0) = 1$, and $\abs{\phi''(\cdot)} \leq c.$ \footnote{This last assumption can also be weakened to say that $\phi''$ is continuous at $0$. See the proof of the theorem for details.} 
Then at initialization as uniformly random orthogonal matrices
\begin{align*}
    \norm{u_{(0)}^\phi - u_{(0)}^{\bbI}} ~ \leq ~ K c \norm{x}^2d\sqrt{\frac{\log^4 h}{h}}
\end{align*}
where $K$ is an universal constant $\phi$, $d$ is the feature dimension and $h$ the width of the hidden layer.
\end{theorem}

We furthermore conjecture that the same behaviour holds during evolution.
\begin{conjecture}[\textbf{Evolution of Non-linear Networks}]\label{conj: non-linear}
Consider the setup of Theorem~\ref{th:non-linear} with the linear $\left( u_{(t)}^{\bbI}\right)$ and non-linear machine $\left( u_{(t)}^\phi\right)$ as defined in \eqref{eq: Th 2 machine u} and an optimization of the general
\begin{align*}
&\min_{W_2W_1}\Tr\left(u_{(t)}^\top  u_{(t)}\right) 
    ~ ~ \st ~ ~ W_2^\top W_2 = \bbI_z \text{ and } W_1^\top W_1 = \bbI_d.\nonumber
\end{align*}
Again assume $\phi$ is an activation function, such that $\phi(0) = 0$ and $\phi'(0) = 1$.
Then 
\begin{align*}
   \norm{ u_{(t)}^\phi - u_{(t)}^{\bbI}} \rightarrow 0\quad \forall t>0 \text{ as } h\rightarrow\infty.
\end{align*}
\end{conjecture}
Numerical justification of the above conjecture is presented in the following section.

\subsection{Numerical Evaluation.} 
We now illustrate the findings of of Theorem~\ref{th:non-linear} and Conjecture~\ref{conj: non-linear} numerically. For evaluation we use the following experimental setup: We train a network with contrastive loss as defined in \eqref{eq:contrastive loss function} using gradient descent with learning rate $0.01$ for $500$ epochs and hidden layer size from $10$ to $2000$. We consider the following three loss functions: (1) sigmoid, (2) ReLU $(\phi(x) = \max\{x,0\})$ and (3) tanh.
The results are shown in Figure~\ref{fig:linear non linear} where the plot shows the average over $10$ initializations.
We note that \emph{tanh} fulfills the conditions on $\phi$ and we see that with increasing layer size the difference between linear and non-linear goes to zero. While \emph{ReLU} only fulfills $\phi(0) = 0$ the overall picture still is consistent with \emph{tanh} but with slower convergence. Finally the results on \emph{sigmoid} (which has a linear drift consistent with its value at $0$) indicate that the conditions on $\phi$ are necessary as we observe the opposite picture: with increased layer width the difference between linear and non-linear increases.

\begin{figure}[t]
    \centering
\includegraphics[width = 0.47\textwidth]{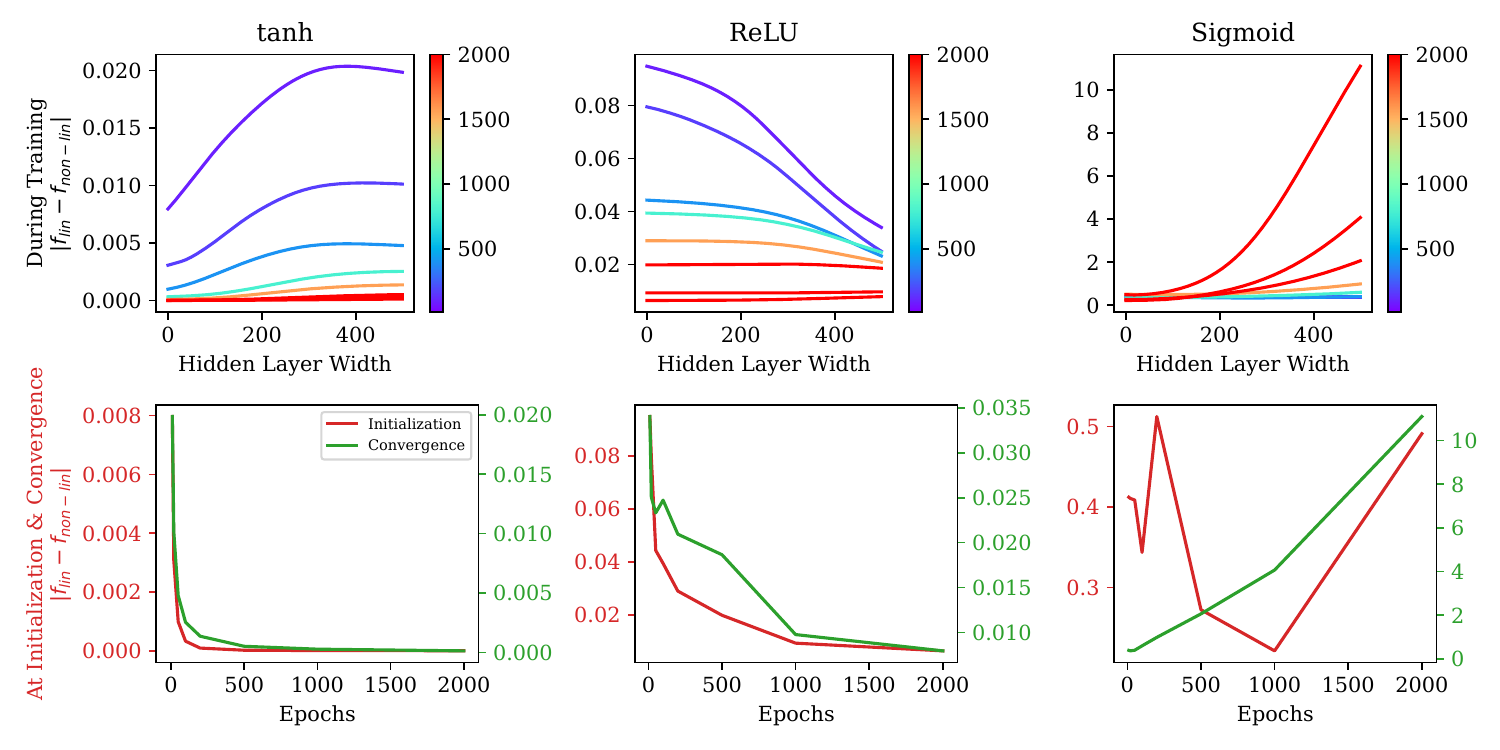}
    \caption{Difference between the non-linear output and the linear output under various conditions on the activation function. 
    \textbf{Row 1.} Change of the difference while training for hidden layer size $10$ to $2000$ (indicated by color bar). 
    \textbf{Row 2.} Difference at initialization and epoch $500$.}   \label{fig:linear non linear}
\end{figure}

\section{Learning Dynamics of Linear SSL Models}\label{sec: linear dynamics}

Having showed that the non-linear dynamics are close to the linear ones we now analyze the linear dynamics. We do so by first showing that the two SSL settings discussed in the introduction can be phrased as a more general trace minimization problem. From there we derive the learning dynamics and discuss the evolution of the differential equation. Furthermore we numerically evaluate the theoretical results and show that the dynamics coincide with learning the general loss function under gradient decent.

We can define a simple linear embedding function $u$ as: $u(x) = W_2^\top W_1x $ where the feature dimension is $d$ for $n$ data points. The hidden layer dimension is $h$ and embedding dimension $z$, such that the weights are given by $W_2\in\bbR^{h\times z}, W_1\in\bbR^{h\times d}$. Therefore we can write our loss function as
\begin{align*}
   \loss
   &=\sum^n_{i=n}\Tr\left(W_2^\top W_1x_i\left(x_i^- - x_i^+\right)^\top W_1^\top W_2\right)\\
   &=\Tr\left(W_2^\top W_1\widetilde{C}W_1^\top W_2\right)
   =\Tr\left(W_2^\top W_1{C}W_1^\top W_2 \right)
\end{align*}
with
\begin{align}\label{eq: simple contrastive C}
    C = \frac{ \widetilde{C} +  \widetilde{C}^\top }{2}
    ~ \text{ and } ~ \widetilde{C}= \sum^n_ix_i\left(x_i^- - x_i^+\right)^\top .
\end{align}
Furthermore \eqref{eq:contrastive loss function} can easily be extended to the $p$ positive and $q$ negative sample setting where we then obtain $\widetilde{C} = \sum_i^n\left( \sum_j^qx_i\left(x_j^-\right)^\top  - \sum_j^px_i\left(x_j^+\right)^\top \right).$
In addition we can also frame the previously considered non-contrastive model in \eqref{eq:non-contrastive loss function} in the simple linear setting by considering the general loss function with 
$\widetilde{C} = \sum_i^n x_i \left( x_i^{+}\right)^\top.$
We can now consider  the learning dynamics of models, that minimize objects of the form
\begin{definition}[General Loss Function]\label{def: general form}
Consider the following loss function
\begin{align}\label{eq: general form}
    &\calL_{W_2W_1}:=\Tr\left(W_2^\top W_1CW_1^\top W_2\right)\\
    &
 ~ ~ \st ~ ~ W_2^\top W_2 = \bbI_z \text{ and } W_1^\top W_1 = \bbI_d.\nonumber
\end{align}
where $W_1\in\bbR^{h\times d}$ and $W_2\in\bbR^{h\times z}$ are the trainable weight matrices. $C\in\bbR^{d\times d}$ is a symmetric, data dependent matrix.
\end{definition}

With the general optimization problem set up we can analyze \eqref{eq: general form} by deriving the dynamics under orthogonality constraints on the weights, which constitutes gradient descent on the Grassmannian manifold.
While orthogonality constraints are easy to initialize the main mathematical complexity arises from ensuring that the constraint is preserved over time. Following \cite{Zehua2020}, we do so by ensuring that the gradients lie in the tangent bundle of orthogonal matrices. 

\subsection{Theoretical Analysis}
In the following we present the dynamics in Theorem~\ref{th: linear dynamics}, followed by the analysis of the evolution of the dynamics in Theorem~\ref{th: convergence}.
\begin{theorem}[\textbf{Learning Dynamics in the Linear Setting}]\label{th: linear dynamics}
Let us recall the the general linear trace minimization problem stated in \eqref{eq: general form}:
\begin{align*}
\min_{W_2W_1}\Tr\left(W_2^\top W_1CW_1^\top W_2\right) 
    ~ ~ \st ~ ~ W_2^\top W_2 = \bbI_z \text{ and } W_1^\top W_1 = \bbI_d.\nonumber
\end{align*}
where $W_1\in\bbR^{h\times d}$ and $W_2\in\bbR^{h\times z}$ are the trainable weight matrices and $C\in\bbR^{d\times d}$ a symmetric, data dependent matrices, such that $C = V{\Lambda}V^\top $ with $V: = \left[v_1,\dots,v_d\right]$. Then with $q:=\left[u^{\bbI}(v_1),\cdots,u^{\bbI}(v_d)\right]^\top $, where $u$ represents the machine function i.e. $u^{\bbI}(x) = W_2^\top W_1 x$, the learning dynamics of $q$, the machine outputs are given by
\begin{align}\label{eq:differential eq}
\mathring{q} = -2\big[2 \Lambda q -  \Lambda q q^\top q  -  q q^\top \Lambda q\big] .
\end{align}
\end{theorem}

Similar differential equations to \eqref{eq:differential eq} have been analysed in \cite{Ojariccati} and \cite{fukumizu1998dynamics}. The typical way to find stable solutions to such equations involve converting it to a differential equation on $qq^\top$. This gives us a matrix riccati type equation. For brevity's sake we write below a complete solution when $z=1$.

\textbf{Evolution of the differential equation.}
While the above differential equation doesn't seem to have a simple closed form, a few critical observations can still be made about it - particularly about what this differential equation converges to. 
As observed in Figure~\ref{fig:NN comparison} (right), independent of initialisation we converge to either of two points. In the following we formalise this observation. 

\begin{theorem}[Evolution of learning dynamics in \eqref{eq:differential eq} for $z=1$]\label{th: convergence}
Let $z=1$ then our update rule simplifies to 
\begin{align}\label{eq: differential rewritten}
    \frac{\mathring{q}}{2} = -(1-q^\top q)\Lambda q-(\bbI-q q^\top)\Lambda q.
\end{align}
We can distinguish two cases:
\begin{itemize}
    \item Assume all the eigenvalues of $\Lambda$ are strictly positive then $q$ converges to $0$.
    \item Assume there is at least one negative eigenvalue of $\Lambda$, then $q$ becomes the smallest eigenvector, $e_1$.
\end{itemize}
\end{theorem}

The requirement of negative eigenvalues of $C$ for a non-trivial convergence might be surprising however we can observe this when considering $C$ in expectation. Let us assume $C$ is constructed by \eqref{eq: simple contrastive C} and note that $\bbE[\widetilde{C}]= \bbE \big[\sum^n_ix_i\left(x_i^- - x_i^+\right)^\top \big].$ While this already gives a heuristic of what is going on, for some more precise mathematical calculations, we can specialise to the situation where $x^-$ is given by an independent sample and $x^+$ is given by adding a noise value $\epsilon$ sampled from $N(0,\sigma \bbI)$, i.e. $x^+ = x +\epsilon$. Then 
\begin{align*}
    \bbE[\widetilde{C}]= \sum^n_{i=1}\bbE [x_i] \bbE[{x_i^-}^\top] - \bbE[x_i{x_i^+}^\top]  
    = - n \bbE[xx^\top ].
\end{align*}
Thus $\bbE[C]$ is in fact negative definite.

\textbf{New Datapoint}.
While the above dynamics provide the setting during training we can furthermore investigate what happens if we input a new datapoint or a testpoint to the machine. Because $u$ is a linear function and because $v_1, ..., v_d$ is a basis this is quite trivial. So if $\hat{x}$ is a new point, let $\alpha = (\alpha_1, ..., \alpha_d)^\top$ be the co-ordinates of $\hat{x}$, i.e. $\hat{x} = \sum_i^d \alpha_i v_i$ or $\alpha = V^\top \hat{x}$. Then \begin{align*}
    u_t(\hat{x}) = u_t\left(\sum_i^d \alpha_i v_i\right) = \sum_i^d \alpha_i u_t(v_i) = q_t^\top \alpha =  q_t^\top V^\top \hat{x}.
\end{align*}

\begin{figure}[t]
    \centering
    \includegraphics[width = 0.47\textwidth]{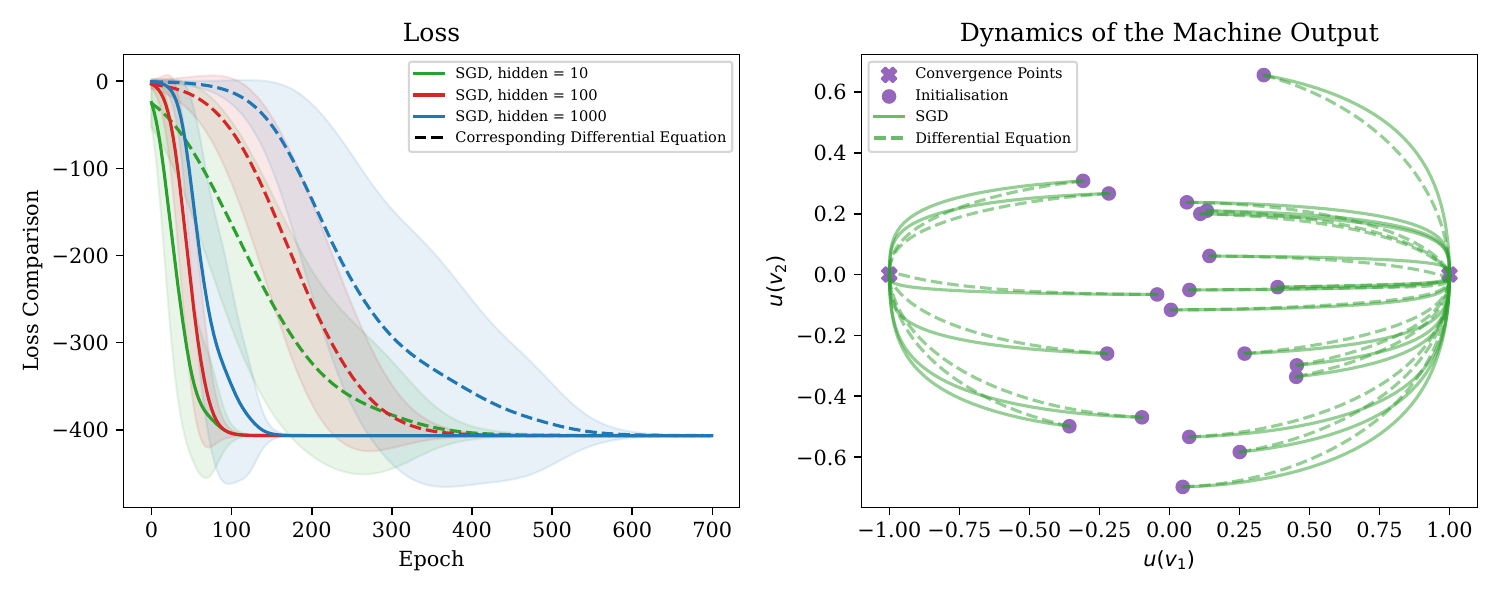}
    \caption{Comparison of gradient decent optimization and differential equation. \textbf{(left)} comparison of the loss function \textbf{(right)} comparison of the outputs.}
    \label{fig:NN comparison}
\end{figure}

\subsection{Numerical Evaluation}

We can now further illustrate the above derived theoretical results empirically.

\textbf{Leaning dynamics (Theorem~\ref{th: linear dynamics}) and new Datapoint.} We can now illustrate that the derived dynamics in \eqref{eq:differential eq} do indeed behave similar to learning \eqref{eq: general form} using gradient decent updates. 
To analyze the learning dynamics we consider the gradient decent update of \eqref{eq: general form}:
\begin{align}\label{eq: SGD}
   W_{1,2}^{(t+1)} =  W_{1,2}^{(t)} + \eta\nabla\calL_{W_2^{(t)}, W_1^{(t)}}
\end{align}
where $W_1^{(t)},W_2^{(t)}$ are the weights at time step $t$  and $\eta$ is the learning rate as a reference. Practically the constraints in \eqref{eq: general form} are enforced by projecting the weights back onto $W_2^\top W_2 =  \bbI_z$ and $W_1^\top W_1 = \bbI_d$ after each gradient step. 
Secondly we consider a discretized version of \eqref{eq:differential eq}
\begin{align}\label{eq:differential disc}
q_{t+1} = q_{t}-2\eta\big[2 \Lambda q_{t} -  \Lambda q_{t} q_{t}^\top q_{t}  -  q_{t} q_{t}^\top \Lambda q_{t}\big] .
\end{align}
where  $q_t$ is the machine outputs at time step $t$. 
We now illustrate the comparison through in Figure~\ref{fig:NN comparison} where we consider different width of the network $(h\in\{10,100,1000\})$ and $\eta = 0.01$. We can firstly observe on the left, that the loss function of the trained network and the dynamics and observe while the decay is slightly slower in the dynamics setting both converge to the same final loss value. 
Secondly we can compare the function outputs during training in Figure~\ref{fig:NN comparison} (right): We initialize the NN randomly and use this initial machine output as $q_0$. We observe that during the evolution using \eqref{eq: SGD} \& \eqref{eq:differential disc} for a given initialization the are stay close to each other and converge to the same final outputs.

\begin{figure}[t]
    \centering
    \includegraphics[width = 0.47\textwidth]{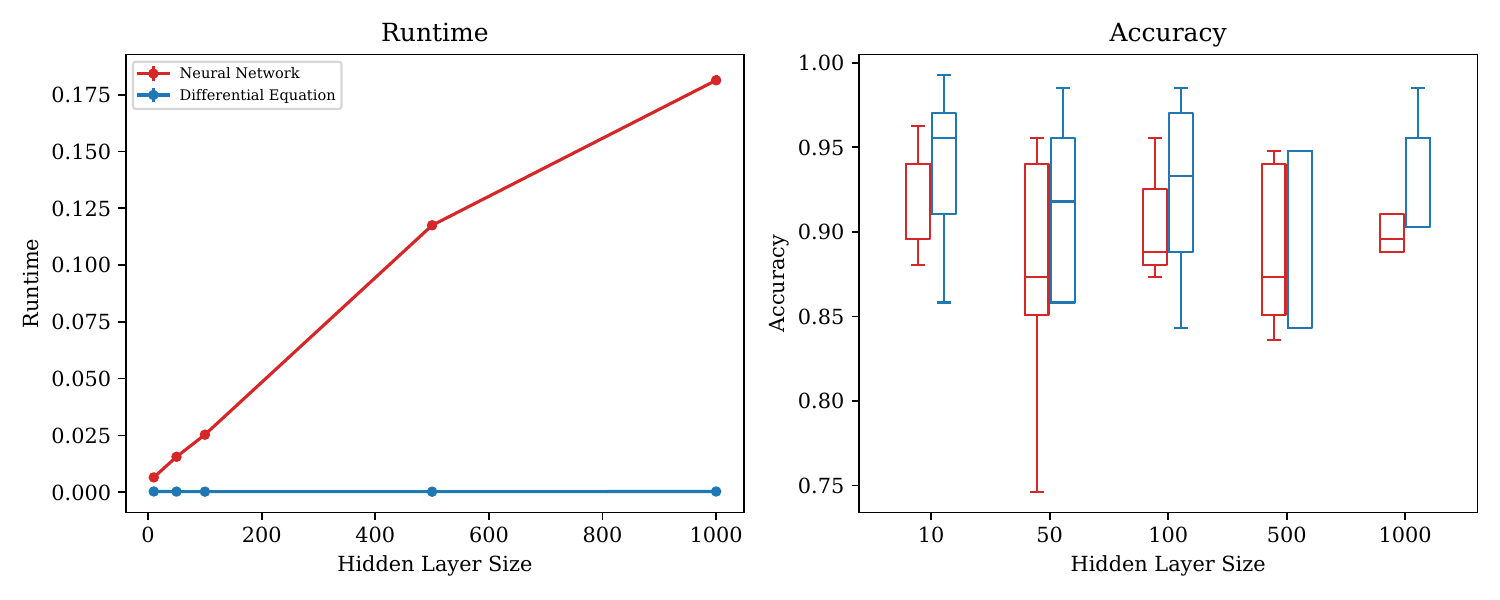}
    \caption{\textbf{(left)} Run-time comparison between running differential equation and SGD iteration for different hidden layer width.
    \textbf{(right)} Downstream task: Accuracy comparison for SVM on embedding obtained by SGD optimization and running the differential equation.}
    \label{fig:runtime}
\end{figure}

\textbf{Runtime and downstram task.} Before going into the illustration of the dynamics we furthermore note that an update step using \eqref{eq:differential disc} is significantly faster then a SGD step using \eqref{eq: SGD}. For this illustration we now consider two classes with 200 datapoints each from the MNIST dataset \cite{deng2012mnist}.  This is illustrated in Figure~\ref{fig:runtime} (left) where we compare the runtime over different layer width (of which \eqref{eq:differential disc} is independent of). Expectantly \eqref{eq: SGD} scales linearly with $h$ and overall \eqref{eq:differential disc} has a shorter runtime per timestep.
While throughout the paper we focus on the obtained embeddings we can furthermore consider the performance of downstream tasks on top of the embeddings. We illustrate this in the setting above where we apply a linear SVM on top of the embeddings. The results are shown in Figure~\ref{fig:runtime} (right) where we observe that overall the performance of the downstream task for both the SGD optimization and the differential equation coincide.

\textbf{Numerical Evaluation of Theorem~\ref{th: convergence}.} We can again illustrate that the behaviour stated in Theorem~\ref{th: convergence} can indeed be observed empirically. This is shown in Figure~\ref{fig:NN comparison} (right), a setting where $C$ has negative eigenvalues. We  observe that eventually the machine outputs converge to the smallest eigenvector.

\section{Conclusion}\label{sec: conclusion}

The study of learning dynamics of (infinite-width) neural networks  has led to important results for the supervised setting. However, there is little understanding of SSL dynamics. Our initial steps towards analysing SSL dynamics encounters a hurdle: standard SSL training has drastic dimension collapse (Proposition \ref{prop:dimension-collapse}), unless there are suitable constraints. 
We consider a general formulation of linear SSL under orthogonality constraints \eqref{eq: general form}, and derive its learning dynamics (Theorem \ref{th: linear dynamics}).
We also show that the derived dynamics can approximate the SSL dynamics using wide neural networks (Theorem \ref{th:non-linear}) under some conditions on activation $\phi$. We not only provide a framework for analysis of SSL dynamics, but also shows how the analysis can critically differ from the supervised setting. As we numerically demonstrate, our derived dynamics can be used an efficient computational tool to approximate SSL models. In particular, the equivalence in Proposition \ref{prop:L2-ortho-equivalence} ensures that the orthogonality constraints can be equivalently imposed using a scaled loss, which is easy to implement in practice.
We conclude with a limitation and open problem. Our analysis relies on a linear approximation of wide networks, but more precise characterisation in terms of kernel approximation \citep{Jacot2018Neurips,0001ZB20} may be possible, which can better explain the dynamics of deep SSL models. However, integrating orthogonality or operator norm constraints in the NTK regime remains an open question.

\section{Acknowledgments}

This work has been supported by the German Research Foundation through the SPP-2298 (project GH-257/2-1), and also jointly with French National Research Agency through the DFG-ANR PRCI ASCAI.

\bibliography{arxiv_bib}

\onecolumn
\section*{Appendix}

~

In the supplementary material we provide the following additional proofs and results
\begin{itemize}
    \item Proof of Lemma\ref{lem:dimension-interaction}
\item Proof of Proposition~\ref{prop:dimension-collapse}
\item Proof of Proposition~\ref{prop:L2-ortho-equivalence}
\item Proof of Theorem~\ref{th:non-linear}
\item Proof of Theorem~\ref{th: linear dynamics}
\item Proof of Theorem~\ref{th: convergence}
\end{itemize}

~

\subsection{Proof of Lemma\ref{lem:dimension-interaction}}
\begin{proof}

Let the collumns of $W_2$ be denoted by $w_1, w_2, ..., w_z$. Then we note that each component of $u$, $u_j$ is given by $u_j(x) = w_j^\top \phi(W_1x)$. Thus if $l \neq j$, $u_j(x)$ has no dependence with $w_l$ i.e. $\nabla_{w_l} u_j(x) = 0$. Thus we get that when $l \neq j$, 
\begin{align*}
    \left\langle\nabla_\Theta u_l(x), \nabla_\Theta u_j(x')\right\rangle = \left\langle\nabla_{W_1} u_l(x), \nabla_{W_1} u_j(x')\right\rangle .
\end{align*}
We can now use \cite{0001ZB20} (for instance its Lemma 1) which basically concludes that no training happens at the penultimate or prior layers. In limit all positive gradients arise only from the final layer. As such \[\left\langle\nabla_{W_1} u_l(x), \nabla_{W_1} u_j(x')\right\rangle = 0.\]

By the same token, for $l=j$, 
\begin{align*}
    \left\langle\nabla_\Theta u_l(x), \nabla_\Theta u_j(x')\right\rangle =& \left\langle\nabla_{W_1} u_l(x), \nabla_{W_1} u_j(x')\right\rangle + \left\langle\nabla_{w_j} u_j(x), \nabla_{w_j} u_j(x')\right\rangle \\
    =& \left\langle \phi(W_1x), \phi(W_1x') \right\rangle.
\end{align*}
Finally again using the fact that $W_1$ does not change in training and that $W_1$ is initialized from a normalized gaussian , when $\phi$ is the identity map, it is well known that the above converges to $x^\top x'$ (as there $\left\langle \phi(W_1x), \phi(W_1x')\right\rangle = x^\top (W_1^\top W_1) x \rightarrow x^\top x'$) and otherwise to a deterministic kernel $k$ (see e.g. \citep{0001ZB20}, \citep{Arora2019ATA}).
\end{proof}

\subsection{Proof of Proposition~\ref{prop:dimension-collapse}} 
\begin{proof}

For simplicity of the proof we begin by reformulating the loss function in both contrastive and noncontrastive setting to a more general form. In particular it is trivial to check that we can generalize by writing \[\loss = \Tr\left(W_2^\top f(X, W_1) W_2 \right),\] where $X$ denotes the collection of all the relevant data (i.e. $\forall~1\leq i\leq n$ $x_i$, as well as $x_i^{+}$ and $x^{-}$ where applicable), and $f(X, W_1) = \sum_{i=1}^n \phi(W_1x_i)\left(\phi(W_1x_i^{-}) - \phi(W_1x_i^+) \right)^\top$ in the contrastive setting (equation~\ref{eq:contrastive loss function}) while  $f(X, W_1) = -\sum_{i=1}^n \phi(W_1x_i)\phi(W_1x_i^{+})^\top$ in the non-contrastive setting (equation~\ref{eq:non-contrastive loss function}.)

Then decompose \[W_2 W_2^\top = \sum_{i=1}^k \sigma_i^2 v_iv_i^\top.\] Note then that $\norm{W_2}_F^2 = \Tr\left(W_2 W_2^\top\right) = \sum_{i=1}^k \sigma_i^2.$ Thus the optimization target, 
\begin{align*}
    \loss(W_1, W_2) &= \Tr\left(W_2^\top f(X,W_1) W_2\right) = \Tr\left(f(X,W_1) W_2W_2^\top\right) = \Tr\left(f(X,W_1)\sum_{i=1}^k \sigma_i^2 v_iv_i^\top \right) \\
    &= \sum_{i=1}^k \sigma_i^2 v_i^\top f(X,W_1) v_i  \geq \min_{i = 1 \textit{ to } k}\{v_i^\top f(X,W_1) v_i\}  \sum_{i=1}^k \sigma_i^2 = \norm{W_2}_F^2 \min_{i = 1 \textit{ to } k}\{v_i^\top f(X,W_1) v_i\}.
\end{align*}
Thus when the Frobenius norm is restricted (i.e. bounded between $0$ and $c$), if $f(X,W_1)$ has atleast one negative eigenvalue the loss is minimized when $v_1$ is the eigenvector corresponding to the most negative eigenvalue of  $f(X,W_1)$ with $\sigma_1 = \norm{W_2}_F$, with no other non-zero singular value.
On the other hand if $f(X,W_1)$ has no negative eigenvalue then the loss is minimized when $W_2=0.$
\end{proof}

\subsection{Proof of Proposition~\ref{prop:L2-ortho-equivalence}}\label{sec: proof normalization}
\begin{proof}
    
We begin by quickly observing that $(1) \iff (2).$ This is simply done by defining $\hat{W}_i = \frac{W_i}{\norm{W_i}_2}$ for $i=1,2$. Then we have 
\[
\argmin_{W_1,W_2} \frac{\Tr\left(W_2^\top W_1{C}W_1^\top W_2 \right)}{\norm{W_1}_2^2\norm{W_2}_2^2} = \argmin_{\hat{W}_1,\hat{W}_2 : \norm{\hat{W}_1}_2 = \norm{\hat{W}_1}_2 = 1}  \Tr\left(\hat{W}_2^\top \hat{W}_1{C}\hat{W}_1^\top \hat{W}_2 \right)
\]
Using the fact that at least one eigenvalue of $C$ is strictly negative (this rules out the case that the optimal is achieved when $W_i = 0$ as that would have prevented division by norm)  then we can quickly get that
\[
\argmin_{\hat{W}_1,\hat{W}_2 : \norm{\hat{W}_1}_2 = \norm{\hat{W}_1}_2 = 1}  \Tr\left(\hat{W}_2^\top \hat{W}_1{C}\hat{W}_1^\top \hat{W}_2 \right)=
\argmin_{\hat{W}_1,\hat{W}_2 : \norm{\hat{W}_1}_2 \leq 1; \norm{\hat{W}_1}_2 \leq 1}  \Tr\left(\hat{W}_2^\top \hat{W}_1{C}\hat{W}_1^\top \hat{W}_2 \right).
\]

For $(2) \iff (3)$, we begin by observing that by submultiplicativity of norm, any $W_1, W_2$ such that $\norm{W_1}_2 \leq 1$ and $\norm{W_2}_2 \leq 1$ automatically falls is the optimization space given by $\norm{W_1^\top W_2} \leq 1$ thus giving one direction of the optimization equivalence for free. For the other side we note that given any $W_1, W_2$ such that $\norm{W_1^\top W_2}_2^2 = \norm{W_1^\top W_2W_2^\top W_1}_2 \leq 1$, we can construct $\hat{W}_1, \hat{W}_2$  such that $\norm{\hat{W}_i}\leq 1$ and $W_1^\top W_2W_2^\top W_1 = \hat{W}_1^\top \hat{W}_2\hat{W}_2^\top \hat{W}_1$. This follows from considering the singular values decomposition of $W_1^\top W_2$, getting $W_1^\top W_2 = U^\top \Sigma V$. As the norm of the product is smaller than $1$, all the entries of the singular value matrix $\Sigma$ are less than $1$. Thus depending upon which among $d$ or $z$ is larger we consider either the matrices $\Sigma U$  and $V$ or the matrices $U$ and $\Sigma V$ to be our candidate $\hat{W}_1$ and $\hat{W}_2$ respectively. To complete we will simply have to add zero rows to our choice i.e. say $U$ and $\Sigma V$ to match the dimensions (i.e. to get a $n \times d$ matrix from a $z \times d$ one).

Finally for $(3) \iff (4)$ we begin by defining $W = W_1^\top W_2$. Then the optimization problem in (3) becomes, 
\[\min_{W : \norm{W}_2 \leq 1} \Tr\left(W^\top C W \right) = \min_{W : \norm{W}_2 \leq 1} \Tr\left( C WW^\top \right).
\]
We then prove that we are done if we can prove the claim at optimal of (3) (i.e. the above optimization problem) all the eigenvalues of $WW^\top$ are $1$ or $0$. Given this claim  the singular value decomposition of $W$ becomes only $W = U^\top V$, where if $k = \textit{rank}(W)$, $U$ is a $k \times d$ matrix and $V$ a $k \times z$ matrix. Additionally by property of SVD, the collumns of $U$ and $V$ are orthonormal. Finally as \[k = \textit{rank}(W) \leq \min\{\textit{rank}(W_1), \textit{rank}(W_2)\} \leq \min\{d,z\} \leq n,\]
we can add a bunch of zero rows to $U$ and $V$ to get our $n \times d$ and $n \times z$ matrices which will be our corresponding $W_1$ and $W_2$.

It remains to prove that $\Tr\left( C WW^\top \right)$ is minimized when all the eigenvalues of $WW^\top$ are $1$ or $0$. To do this simply decompose \[WW^\top = \sum_{i=1}^{k} \sigma_i^2 v_iv_i^\top,\] where $v_i$ is the set of orthonormal eigenvectors of $WW^\top$ corresponding to non-zero eigenvalues of $WW^\top$ (or alternatively non-zero singular values of $W$)
Then \begin{align*}
    \Tr\left( C WW^\top \right) =& \Tr\left( C \sum_{i=1}^{k} \sigma_i^2 v_iv_i^\top \right) 
    \\=& \sum_{i=1}^{k} \sigma_i^2 \Tr\left( C   v_iv_i^\top \right)
    \\=& \sum_{i=1}^{k} \sigma_i^2 v_i^\top C v_i.
\end{align*}

Thus if $C$ has $l$ many strictly negative eigenvalues $\lambda_1 \leq \dots \leq  \lambda_l$ with corresponding eigenvectors $c_1, \dots , c_l$ and $\sigma_i^2$ is positive the above quantity is minimized by choosing as many of these as possible i.e. $v_1 =  c_1, \dots, v_{\min\{d,z,l\}} = c_{\min\{d,z,l\}}$ and setting the corresponding $\sigma_i$ to be $1$ while every setting all other eigen-values to $0$.

We then also note by consequence of the above proof that we avoid dimension collapse when possible i.e. when $C$ has multiple strictly negative eigenvalues (which is what one should expect if the data is not one dimensional as $\bbE[C] = - \bbE[xx^\top]$)
\end{proof}

\subsection{Proof of Theorem~\ref{th:non-linear}} \label{sec: proof th non linear}
\begin{proof}
Let us start by defining some properties for the non-linearity:
 Assume the non-linear function $\phi$ is continuously twice differentiable near $0$ and has no bias i.e. $\phi(0) = 0.$ Then via scaling we can assume WLOG that $\phi'(0) = 1$. As $\abs{\phi''(x)} \leq c$, we get that \footnote{We can actually also use the weaker assumption that $\phi''(0)$  is continuous at $0$. Thus there is some bounded (compact) set $A$ containing $0$ and a constant $c$ such that $\forall x \in A$, $\abs{\phi(x) - x} \leq \frac{cx^2}{2}$}
    \begin{align} \label{eq:taylor}
        \abs{\phi(x) - x} \leq \frac{cx^2}{2}.
    \end{align}
    Recall that the mapping of the first weight matrix is given by
    $
        W_1:\bbR^d\rightarrow \bbR^h,\quad h\gg d 
    $
    under the constraint that $W_1^\top W = \bI$. Under uniformly random initialization by Lemma~\ref{lem:mxconc} (see proof below) then with probability asymptotically going to $1$ we have that \[\max{(W_1)^2_{i,j}} \leq C\frac{\log^2 h}{h}\]
    Thus the norm of each row of $W_1$ we get with a.w.h.p. :
    \[\norm{\row_i\left(W_1\right)}^2 = \sum_{j=1}^d (W_1)^2_{i,j} \leq C\frac{d\log^2 h}{h}\]
    From there we can now write the value of each node in the layer using Cauchy-Schwarz inequality as
    \begin{align} \label{eq:CSIsq}
\left| \row_i{(W_1)}\cdot x \right|^2
\leq\norm{\row_i\left(W_1\right)}^2\norm{x}^2
\leq C\norm{x}^2 \frac{d\log^2 h}{h}.
    \end{align}
    We now apply the non-linearity to this quantity and denote the output of the first layer after the non-linearity as
    \begin{align*}
        v_i &= \phi\left(\row_i\left(W_1\right)\cdot x\right)
    \end{align*}
   Define the vector $\epsilon \in \bbR^h$, where \[\epsilon_j = v_i - \row_i\left(W_1\right)\cdot x\]
   Then we have for $h$ large enough\footnote{Note that for the weaker assumption we can still use equation~\ref{eq:taylor}. This is because by equation~\ref{eq:CSIsq},w.h.p. $\row_i{(W_1)}\cdot x$ goes to $0$ and thus $\row_i{(W_1)}\cdot x \in A$ in limit}: 
   \begin{align*}
           \norm{\epsilon}^2 &= \sum_{i=1}^h \epsilon_i^2 \\
           &= \sum_{i=1}^h (v_i - \row_i\left(W_1\right)\cdot x)^2 \\
           &\le \sum_{i=1}^h \frac{c^2}{4}\left(\row_i\left(W_1\right)\cdot x\right)^4 
           &\textit{by equation \ref{eq:taylor}} \\
           &\le \sum_{i=1}^h \frac{c^2}{4}\left(C\norm{x}^2 \frac{d\log^2 h}{h}\right)^2 &\textit{by equation \ref{eq:CSIsq}} \\
           &= K^2c^2\norm{x}^4\frac{hd^2\log^4 h}{h^2} = K^2c^2\norm{x}^4\frac{d^2\log^4 h}{h} ,
    \end{align*} 
    where $K$ is the universal constant $\frac{C}{2}.$
Combining this with the second layer we get the difference of the outputs of the two networks as
\begin{align*}
    \norm{u_{(0)}^\phi - u_{(0)}^{\bI}} &= \norm{W_2^\top v  - W_2^\top W_1x}\\
    &= \norm{W_2^\top \left(v  - W_1x\right)} \\
    &\leq \norm{W_2}\norm{\epsilon} 
    = \norm{\epsilon} &\textit{as $\norm{W_2} =1$}\\
    &\leq Kc\norm{x}^2d \sqrt{\frac{\log^4 h}{h}} \\
    &\rightarrow 0.
\end{align*}
\end{proof}

\begin{lemma} \label{lem:mxconc}
    Given any $d \leq p$, Let $Q$ be a uniformly random $h \times d$ semi-orthonormal matrix. I.e. $Q$ is the first $d$ columns of an uniformly random $h \times h$ orthonormal matrix. Then there are constants $L$ and a sequence $\epsilon_p$ converging to $0$ as $h$ goes to infinity such that , 
    \[
      P\left(\max \abs{Q_{i,j}} \geq \frac{L\log h}{\sqrt{h}}\right) \leq \epsilon_n
    \]
\end{lemma}

\begin{proof}
    We note that it is enough to prove the claim when $d=h$, i.e. $Q$ is uniformly random $h \times h$ orthonormal matrix. Then as our distribution is uniform, the density at any particular $Q$ is same as the density at any $UQ$ where $U$ is some other fixed orthogonal matrix. Thus if $q_1$ is the first column of $Q$, the marginal distribution of $q_1$ has the property that its density at any $q_1$ is same as that of $U q_1$ for any orthogonal matrix $U$. In other words the marginal distribution for any column of $Q$ is simply that of the uniform unit sphere.

    Consider then the following random variable which has the same distribution as that of a fixed column of $Q$ i.e. uniform unit $h$-sphere. Let $X = (X_1, ..., X_h)$ be iid random variables from $\calN(0,1)$. Then we know that $X \sim \calN(0,\bI_h)$. From the rotational symmetry property of standard gaussian then we have that $\frac{X}{\norm{X}}$ is distributed as an uniform sample from the unit sphere in $h$ dimensions. By union bound then, we have 
    \begin{align*}
        &P\left(\max_{1 \leq i \leq h} \abs{X_i} \geq t\log h\right) \leq \frac{1}{\sqrt{2\pi}}he^{-\frac{t^2\log^2 h}{2}} \\
        \implies &P\left(\max_{1 \leq i \leq h} \abs{X_i} \leq t\log h\right) \geq 1 - \frac{1}{\sqrt{2\pi}}he^{-\frac{t^2\log^2 h}{2}}.
    \end{align*}
    As each $X_i$ is iid normal, $X_i^2$ is iid Chi-square with $\bbE[X_i^2] = 1$, thus by Chernoff there exists constants $C',c'$ such that \[P\left(\frac{\sum_{i=1}^h X_i^2}{h} \geq 1-s  \right) \geq 1 - C'e^{-c'hs^2}.  \]
    Since $\max_{1 \leq i \leq h} \abs{X_i} \leq t\log h$ and $\frac{\sum_{i=1}^h X_i^2}{h} \leq (1+s)$ implies that $\max_{1 \leq i \leq h} \frac{\abs{X_i}}{\norm{X}} \leq \frac{t \log h}{\sqrt{h(1-s)}}$, we get that 
    \begin{align*}
        &P\left(\max_{1 \leq i \leq h}\frac{\abs{X_i}}{\norm{X}} \leq \frac{t \log h}{\sqrt{h(1-s)}} \right) \geq 1 - \frac{1}{\sqrt{2\pi}}he^{-\frac{t^2\log^2 h}{2}}  - C'e^{-c'hs^2} \\
        \implies &P\left(\max_{1 \leq i \leq h} \frac{\abs{X_i}}{\norm{X}} \geq \frac{t \log h}{\sqrt{h(1-s)}} \right) \leq \frac{1}{\sqrt{2\pi}}he^{-\frac{t^2\log^2 h}{2}}  + C'e^{-c'hs^2}
    \end{align*}
    From the argument before that any $j$'th column of $Q$ is distributed as $X$. Using the above and another union bound then get us
    \begin{align*}
        &P\left(\max_{1 \leq i \leq h} \max_{1 \leq i \leq h} \abs{Q_{i,j}} \geq \frac{t \log h}{\sqrt{h(1-s)}} \right) \leq \frac{1}{\sqrt{2\pi}}he^{-\frac{t^2\log^2 h}{2}}  + C'e^{-c'hs^2} \\
        \implies &P\left( \max_{1 \leq j \leq h} \max_{1 \leq i \leq h} \abs{Q_{i,j}} \geq \frac{t \log h}{\sqrt{h(1-s)}} \right) \leq \frac{1}{\sqrt{2\pi}}h^2e^{-\frac{t^2\log^2 h}{2}}  + C'he^{-c'hs^2}
    \end{align*}
    We note that for any constants $t,c'$ that as $h$ goes to infinity, both $h^2e^{-\frac{t^2\log^2 h}{2}}$ and $he^{-c'hs^2}$ goes to zero. The proof is then finished by choosing some appropriate constants $s,t \geq 0$.
\end{proof}

 \subsection{Proof of Theorem~\ref{th: linear dynamics}}\label{sec: proof linear}
\begin{proof}
To simplify notation we are dropping the superscript $\bbI$ from $u_{(t)}^{\bbI}$. The $u$ in the following proof is already presumed to be linear. For the same reason we are also dropping the symbol of time,  $t$, from $u,W_2,W_1$ even though all of them are indeed time dependent. Finally for any time dependent function $f$, we denote $\frac{\partial f}{\partial t}$ by $\mathring{f}$. 

From \citep{edelman1998geometry}, we get that the derivative of a function $\gamma$ restricted to a grassmanian is derived by left-multiplying $1-\gamma \gamma^\top$ to the "free" or unrestricted derivative of $\gamma$. Using this and recalling that the loss in Eq.~\ref{eq: general form} is given by 
\begin{align*}
 \loss =   \Tr\left(W_2^\top W_1CW_1^\top W_2\right),
\end{align*}
we therefore can write $\mathring{W_1}$ and $\mathring{W_2}$ as 
\begin{align*}
    \mathring{W}_2(t)
    &=-\left(\bI - W_2W_2^\top \right)\nabla_{W_2}\loss
    =-2\left(\bI - W_2W_2^\top \right)\left(W_1{C}W_1^\top W_2\right)\\
    \mathring{W}_1(t)
    &=-\left(\bI - W_1W_1^\top \right)\nabla_{W_1}\loss
    =-2\left(\bI - W_1W_1^\top \right)\left(W_2W_2^\top W_1C\right).
\end{align*}
Thus we obtain
\begin{align*}
     \frac{\partial~u_{(t)}(x)}{\partial~t}  =& \mathring{W}_2(t)^\top W_1(t)x +  W_2(t)^\top \mathring{W}_1(t)x \\
    =&\left(\left(\bI - W_2W_2^\top \right)\left(-2W_1{C}W_1^\top W_2\right)\right)^\top W_1(t)x \\
    &+  W_2(t)^\top \left(\bI - W_1W_1^\top \right)\left(-2W_2W_2^\top W_1C\right)x\\
     =& -2\left(
 W_2^\top W_1C\cancel{W_1^\top W_1}x 
 ~ + ~ \cancel{W_2^\top W_2}W_2^\top W_1Cx \right)\\
 & +2\left(W_2^\top W_1{C}W_1^\top W_2W_2^\top W_1x ~ + ~
 W_2^\top  W_1W_1^\top W_2 ~ W_2^\top W_1Cx\right)\\    
 =& -2\left( 2W_2^\top W_1Cx - W_2^\top W_1{C}W_1^\top W_2W_2^\top W_1x ~ - ~
 \sum^d_iW_2^\top  W_1 v_i v_i^\top  W_1^\top W_2 ~ W_2^\top W_1Cx\right),
\end{align*}
where we obtain the second equality by expanding the terms, taking advantage of $W_2^\top W_2 = \bI, W_1^\top W_1 = \bI$ and $\bI_d = \sum_i^dv_iv_i^\top $.
Now setting $x$ as $v_j$ and using the fact that they are eigenvectors for $C$ and using $ C = \sum_i^d \lambda_i v_iv_i^\top $  gives us:
\begin{align*}
\mathring{u}(v_j) = & -2\left(2\lambda_j u_{(t)}(v_j) -  \sum^d_i \lambda_i u_{(t)}(v_i) u_{(t)}(v_i)^\top u_{(t)}(v_j) - \lambda_j \sum^d_i u_{(t)}(v_i) u_{(t)}(v_i)^\top u_{(t)}(v_j) \right)
\end{align*}
Let's rewrite this in matrix notation. First define $q: = \left[ u(v_1), \dots  u(v_d)\right]^\top $ thus obtaining: 
\begin{align*}
 \mathring{q} = -2\big[2 \Lambda q -  \Lambda q q^\top q  -  q q^\top \Lambda q\big]
\end{align*}
which concludes the proof.
\end{proof}

\subsection{Proof of Theorem~\ref{th: convergence}}
\begin{proof}

For instance first suppose that all the eigenvalues of $\Lambda$ are strictly positive and thus $q^\top \Lambda q > 0$. Then 
\begin{align*}
    \frac{d (q^\top q)}{dt} &= 2 q^\top\mathring{q} = 4\big[-(1-q^\top q)q^\top\Lambda q -q^\top(\bI-q q^\top)\Lambda q\big] \\
    &= -8(1-q^\top q) q^\top \Lambda q
\end{align*}

Observing now that because of orthonormality of our weight matrices, $q^\top q = \norm{q}^2 \leq 1$ we get that the derivative of $\norm{q}^2$ is always negative and thus $q$ converges to $0$. 

Now suppose on the other hand there is atleast one negative eigenvalue. WLOG let $e_1$ denote the eigenvector  with the smallest eigenvalue (which is negative). Then 
\begin{align*}
    \frac{d (e_1^\top q)}{dt} &= e_1^\top\mathring{q} = 2\big[-(1-q^\top q)e_1^\top\Lambda q -e_1^\top(\bI-q q^\top)\Lambda q\big] \\
    &= 2\big[(1-q^\top q)(-\lambda_1) e_1^\top q + (q^\top \Lambda q - \lambda_1)e_1^\top q)\big]
\end{align*}
We now note that $q^\top \Lambda q - \lambda_1 \geq 0$ as $\lambda_1$ is the smallest eigenvalue. Thus as $-\lambda_1$ is positive, the derivative of $e_1^\top q$ is always positive unless $1-q^\top q = q^\top \Lambda q - \lambda_1 = 0$, which only happens at $q = e_1$. In other words, eventually $q$ becomes the smallest eigenvector $e_1$.
\end{proof}
\end{document}